\documentclass[conference]{IEEEtran}

\usepackage[utf8]{inputenc}
\usepackage[T1]{fontenc}
\usepackage{amsmath,amssymb,amsfonts}
\usepackage{amsthm}
\usepackage{mathtools}
\newtheorem{theorem}{Theorem}
\newtheorem{proposition}{Proposition}

\newtheorem{remark}{Remark}
\usepackage{bm}
\usepackage{booktabs}
\usepackage{enumitem}
\usepackage{cite}
\usepackage{hyperref}
\usepackage{xcolor}
\usepackage{tikz}
\usetikzlibrary{arrows.meta, positioning, fit, calc}

\newcommand{\SOthree}{\mathrm{SO}(3)}
\newcommand{\SEthree}{\mathrm{SE}(3)}
\newcommand{\sothree}{\mathfrak{so}(3)}
\newcommand{\Sph}{\mathbb{S}}

\newcommand{\R}{\mathbb{R}}
\newcommand{\norm}[1]{\lVert #1 \rVert}
\newcommand{\abs}[1]{\lvert #1 \rvert}
\newcommand{\diag}{\mathrm{diag}}
\newcommand{\tr}{\mathrm{tr}}
\newcommand{\argmin}{\operatornamewithlimits{argmin}}

\newcommand{\hatmap}[1]{\widehat{#1}}
\renewcommand{\eqref}[1]{Eq. (\ref{#1})}

\begin{document}
\IEEEoverridecommandlockouts 

\title{Decentralized Geometric Control for Cable-Suspended\\Payload Transport with Adaptive Mass Estimation}

\author{%
Hadi Hajieghrary$^{1}$,
Benedikt Walter$^{2}$,
Paul Schmitt$^{3}$,
and Miguel Hurtado$^{4}$%
\thanks{%
$^{1}$Hadi Hajieghrary (\texttt{hadi.hajieghrary@gatech.edu}),
$^{2}$Benedikt Walter (\texttt{walter.benedikt@gmail.com}),
$^{3}$Paul Schmitt (\texttt{pauls@massrobotics.org}), and
$^{4}$Miguel Hurtado (\texttt{miguel.hurtadomit@gmail.com})
prepared this manuscript and present it solely in their individual capacities. The views expressed in this paper are those of the authors and do not necessarily reflect the views of their employers or affiliated organizations.}%
}
\maketitle

\begin{abstract}
Cooperative aerial transport requires controllers that respect nonlinear manifold geometry, operate without centralized coordination, and respect operational safety constraints. To address these demands, we present GPAC, a four-layer hierarchical architecture enabling $N$ quadrotors to transport a cable-suspended payload without a central coordinator and without exchanging cable states or adaptive parameters. The key insight is implicit coordination: each quadrotor independently estimates its effective load share from local cable measurements, so combined forces converge to the correct total, even without knowledge of $N$ or the payload mass; the payload position is reconstructed locally from each agent's own cable geometry, and the only inter-agent communication is a low-rate neighbor-position broadcast for collision avoidance. GPAC operates directly on the full nonlinear configuration manifold and integrates geometric position and attitude control, anti-swing regulation, an extended-state observer for wind rejection, concurrent learning-based mass estimation without persistent excitation, and a priority-ordered control barrier function (CBF)-inspired safety filter that reduces operational risk, with input-to-state safety (ISSf) margins that hold exactly under single-constraint activation. A compatibility result shows that the filter's force modifications keep the desired attitude within the almost-global stability region of the $\mathrm{SO}(3)$ attitude controller. Finally, high-fidelity simulation with flexible cables, onboard sensor fusion, and wind turbulence---with all control and estimation loops closed through the estimator---yields a 33.8 cm mean payload-tracking RMSE (2.8\% coefficient of variation over 13 seeds) at a low per-agent computational cost.
\end{abstract}


\section{Introduction}\label{sec:intro}
Cooperative aerial transport involves several unmanned aerial vehicles (UAVs) working together to carry a payload with cables. This approach offers greater capacity, a larger workspace, and better fault tolerance than using a single UAV. However, these systems are safety-critical. Issues such as cable slack, excessive swing, vehicle tilt, collisions between UAVs, and disturbances can all cause loss of control or make the payload unstable. To achieve cooperation among $N$ quadrotors, controllers must handle the nonlinear system dynamics, operate without centralized coordination, and address the main failure modes.

Lee, Sreenath, and Kumar~\cite{lee2010geometric, sreenath2013geometric} developed a method for cable-suspended transport using the full nonlinear configuration manifold, achieving almost-global stability without issues arising from Euler-angle singularities. Later, they added anti-swing control for the cables~\cite{lee2018geometric}. Sharma and Sundaram~\cite{sharma2023geometric} created a geometric controller for multi-UAV payload transfer that does not need link information. Sun et al.~\cite{sun2025agile} showed agile cooperative cable manipulation with online kinodynamic planning. A differential-geometric constrained-mechanics framework earlier modeled cooperative cable towing---holonomic cable constraints coupled to nonholonomic vehicle dynamics---and decomposed the team transport problem into agent--load subsystems, validated on planar marine vehicles~\cite{hajieghrary2018differential}; GPAC carries this geometric decomposition into the aerial, safety-critical setting, replacing shared payload-state feedback with local cable-based reconstruction and adaptive load-share estimation. However, these controllers require centralized state information or complete feedback on the payload. Each quadrotor must know the number of cooperating UAVs~$N$ and the payload mass~$m_L$, or the states of all other UAVs. This is not practical if communication is unreliable.

Consensus-based formation controllers and distributed optimization methods use linearized dynamics and Euclidean error measures. These methods lose global stability during large-angle maneuvers, which is when stability is most important. Wang et al.~\cite{wang2024automultilift} introduced Auto-Multilift, a distributed learning system that tunes model predictive control (MPC) cost functions online for cooperative load transport. However, this method relies on optimization-based MPC rather than geometric control and does not provide Lyapunov stability guarantees. Concurrent learning~\cite{chowdhary2010concurrent, chowdhary2013exponentially} allows convergence without persistent excitation by using stored data, which matters because cooperative hover does not provide enough excitation. So far, concurrent learning has only been used in centralized, single-vehicle cases.

Cable-suspended transport requires real-time state constraints. The cables must stay taut, cable angles must be within limits, swing rates must be controlled, and collisions must be avoided. Control Barrier Functions (CBFs)~\cite{ames2017control, ames2019control} enforce these constraints using online quadratic programs (QPs) that make minimal changes to the main controller. Yang and Xie~\cite{yang2025robust} combined CBFs with disturbance estimators to improve safety for single-quadrotor cable-suspended payloads. They found that disturbance-observer CBFs are less conservative when the model is uncertain. However, combining CBFs with geometric controllers on nonlinear manifolds while maintaining Lyapunov guarantees remains an open problem. This is even harder in decentralized multi-agent systems, where constraints and attitude stability must be managed simultaneously.

Most studies on geometric transport assume cables are rigid and state feedback is perfect. In reality, cables are flexible, show wave effects, and can go from slack to taut. These factors have a big impact on system behavior~\cite{williams2009dynamics}. Onboard estimation must combine noisy IMU and GPS data with nonlinear filters~\cite{sola2017quaternion}. Wind also adds unpredictable forces. How well geometric cooperative controllers handle these real-world issues remains poorly understood. In short, there is no current framework that provides geometric manifold control, decentralized adaptive estimation without persistent excitation, and runtime safety supervision for multi-UAV cable transport.

To close these gaps, safety-critical engineering focuses on identifying and reducing major hazards early~\cite{STPAleveson}. Rather than building a single large controller and checking safety afterward, each GPAC layer targets a specific transport risk: anti-swing control stabilizes swinging, adaptive estimation handles load uncertainty, the extended state observer (ESO) handles disturbances, and the CBF layer acts as a safety supervisor during operation. The system uses a layered, multi-rate structure with separate timescales to limit fault spread, and its coordinator-free design removes single points of failure~\cite{koopman2017autonomous}. Table~\ref{tab:failure_modes} in Section~\ref{sec:results} shows this mapping with real data.

Building on this idea, the GPAC architecture breaks down the cooperative transport problem into $N$ identical single-agent tasks. The key insight is that each quadrotor can estimate its effective load share---its payload share plus the weight of its own cable---from only its local cable measurements. The combined forces from all agents then sum to the total weight the team supports:
\begin{equation}
  \sum_{i=1}^{N} F_i = \sum_{i=1}^{N} \hat{\theta}_i \cdot u \;\longrightarrow\; (m_L + m_{\text{cab}}) \cdot u,
  \label{eq:force_convergence}
\end{equation}
In this setup, $u := g\,e_3 + \ddot{p}_L^d$ is the common payload acceleration demand computed identically by every agent from the shared reference, and $m_{\text{cab}}$ is the total cable mass (the idealized massless-cable limit recovers $m_L\,u$; Section~\ref{sec:estimation}). This approach means there is no need for a central coordinator, for exchanging payload or cable states, for consensus protocols, or for knowing the payload mass in advance. The main contributions are:
\begin{enumerate}[leftmargin=*, itemsep=2pt]
  \item A coordinator-free controller for multi-UAV cable transport that operates directly on the nonlinear manifold without linearization, requiring no payload-state, cable-state, or adaptive-parameter exchange; each agent reconstructs the payload position locally from its own cable, and the only inter-agent communication is a low-rate neighbor-position broadcast for collision avoidance.

  \item Each agent estimates its effective load share (payload plus its own cable) $\hat{\theta}_i$ from local cable tension and direction, with exponential convergence to a uniformly ultimately bounded neighborhood---even during near-hover conditions where classical adaptive laws fail.

  \item A modular, priority-ordered safety-supervision layer that reduces the risk of cable slack, excessive cable angle, excessive tilt, swing-rate growth, and inter-agent collision. The associated input-to-state safety (ISSf) margins hold exactly under single-constraint activation, while simultaneous activation is handled by priority-ordered feasibility. Theorem~\ref{thm:compatibility} shows the CBF-induced force modifications keep the desired attitude within the almost-global stability region of the $\SOthree$ geometric attitude controller.

  \item High-fidelity Drake-based simulation with bead-chain cables, onboard sensor fusion, and Dryden wind turbulence, with all loops closed through the ESKF, achieving $33.8$\,cm mean payload RMSE ($2.8\%$ CV over 13 seeds) at low per-agent computational cost.
\end{enumerate}

\section{System Modeling}\label{sec:modeling}
The world frame~$\mathcal{W}$ has $e_3$ upward. Rotations lie in $\SOthree$; the hat map $(\cdot)^{\wedge}:\R^3 \to \sothree$ satisfies $\hatmap{v}w = v \times w$. Cable directions $q_i \in \Sph^2$ have tangent projection $P(q) = I_3 - qq^\top$. The system evolves on
\begin{equation}
  \mathcal{Q} = \underbrace{\SEthree}_{\text{payload}} \times \prod_{i=1}^{N}\!\underbrace{\SEthree \times \Sph^2}_{\text{quadrotor }i\text{ + cable}}, \quad \dim(\mathcal{Q}) = 6 + 8N.
  \label{eq:manifold}
\end{equation}

$N$~identical quadrotors (mass $m_Q$, inertia $J$) obey
\begin{align}
  m_Q \ddot{p}_i &= -m_Q g\, e_3 + f_i R_i e_3 + F_i^{\text{cable}} + F_i^{\text{wind}}, \label{eq:quad_trans}\\
  J\dot{\Omega}_i &= -\Omega_i \times J\Omega_i + \tau_i + \tau_i^{\text{ext}}, \label{eq:quad_rot}
\end{align}
with kinematics $\dot{R}_i = R_i \hatmap{\Omega}_i$, where $f_i$ is scalar thrust, $\tau_i$ is control torque, $F_i^{\text{cable}}$ is cable tension, and $F_i^{\text{wind}}$ is aerodynamic disturbance. With $e_3$ upward, gravity acts as $-g e_3$ on both the quadrotors and the payload. The payload (mass $m_L$, position $p_L$) satisfies $m_L \ddot{p}_L = -m_L g\, e_3 + \textstyle\sum_{i=1}^{N} F_i^L + F^{\text{contact}} + F_L^{\text{wind}}$, where $F_i^L$ is the cable force from the $i$-th rope at the payload attachment point, and $F^{\text{contact}}$ denotes ground normal and Coulomb friction forces.

The cable direction $q_i \in \Sph^2$ from payload to quadrotor evolves as $\dot{q}_i = \omega_{q_i} \times q_i$ with $\omega_{q_i} \in T_{q_i}\Sph^2$. Under constraint $p_i = p_L + R_L\rho_i^L + L_i q_i$ ($\rho_i^L$: attachment offset in payload frame), swing dynamics on $\Sph^2$ are governed by projected quadrotor acceleration and gravitational restoring torque (the singularity-free error functions parameterizing these dynamics are defined in Section~\ref{sec:control}).

Rigid-link cable models neglect compliance, wave propagation, and slack-to-taut transitions that significantly affect transport behavior. To capture these effects, each cable is discretized as $n_b$ point-mass beads connected by $n_b+1$ tension-only spring-damper segments.

Let $b_0$ be the quadrotor attachment, $b_1,\ldots,b_{n_b}$ the bead positions, and $b_{n_b+1}$ the payload attachment. Each segment has rest length $L_0 = L_{\text{rest}}/(n_b+1)$, stretch $\Delta_j = \norm{b_{j-1}-b_j} - L_0$, and unit direction $\hat{e}_j = (b_{j-1}-b_j)/\norm{b_{j-1}-b_j}$. The tension-only force law is $T_j = k_s \Delta_j + c_s [\dot{\Delta}_j]^+$ when $\Delta_j > 0$, and $0$ when $\Delta_j \leq 0$, where $[\cdot]^+ = \max(\cdot,0)$ restricts damping to stretching. Segment stiffness derives from a maximum-stretch criterion ($\epsilon_{\max} = 15\%$):
\begin{equation}
  k_s = \frac{F_{\text{load}}}{L_{\text{rest}}\,\epsilon_{\max}} \cdot (n_b + 1), \quad c_s = c_{\text{ref}}\sqrt{k_s/k_{\text{ref}}},
  \label{eq:stiffness}
\end{equation}
with $F_{\text{load}} = m_L g/N$ and $k_{\text{ref}} = 300$\,N/m, $c_{\text{ref}} = 15$\,N$\cdot$s/m. Each bead (mass $m_b = m_{\text{rope}}/n_b$, $m_{\text{rope}} = 0.2$\,kg) obeys $m_b\ddot{b}_j = F_j - F_{j+1} - m_bg\,e_3$. The model supports wave propagation ($\sim\sqrt{k_s L_0/m_b}$), distributed inertia, and impulsive loading during slack-to-taut transitions.

Each quadrotor carries an onboard sensor suite (see Table~\ref{tab:params} for parameters). The \emph{IMU} provides body-frame specific force $\tilde{a}_i = R_i^\top(\ddot{p}_i + ge_3) + b_{a_i} + n_{a_i}$ and angular velocity $\tilde{\omega}_i = \Omega_i + b_{g_i} + n_{g_i}$ with Gauss--Markov biases ($\dot{b} = -b/\tau + \eta$, $\tau = 3600$\,s). The \emph{GPS} provides position with dropout probability. Cable tension $T_i$ is measured via a load cell, and direction $q_i$ via a 2-axis encoder.

Wind follows the Dryden turbulence spectrum (MIL-F-8785C) with forming filter $H_\alpha(s) = \sigma_\alpha\sqrt{2V/L_\alpha}/(s + V/L_\alpha)$ and altitude-dependent scaling $\sigma_\alpha(h) \propto (h/h_{\text{ref}})^{1/6}$. Inter-agent spatial correlation decays as $\rho(p_i,p_j) = \exp(-\norm{p_i - p_j}/\ell_c)$ with $\ell_c = 10$\,m, so closely spaced quadrotors experience similar wind.

\begin{table}[t]
  \centering
  \caption{Key System Parameters}
  \label{tab:params}
  \begin{tabular}{@{}lcc@{}}
    \toprule
    \textbf{Parameter} & \textbf{Symbol} & \textbf{Value} \\
    \midrule
    Quadrotor mass / count & $m_Q$ / $N$ & 1.5\,kg / 3 \\
    Payload mass / radius  & $m_L$ / $r_L$ & 3.0\,kg / 0.15\,m \\
    Formation radius       & $r_f$ & 0.6\,m \\
    Beads per cable        & $n_b$ & 8 \\
    Rope mass / max stretch & $m_{\text{rope}}$ / $\epsilon_{\max}$ & 0.2\,kg / 15\% \\
    IMU rate / noise       & $f_{\text{IMU}}$ / $\sigma_a$ & 200\,Hz / $0.004$\,m/s$^2$/$\!\sqrt{\text{Hz}}$ \\
    GPS rate / noise       & $f_{\text{GPS}}$ / $\sigma_{xy}$ & 10\,Hz / 0.02\,m \\
    Baro rate / noise      & $f_{\text{baro}}$ / $\sigma_w$ & 25\,Hz / 0.3\,m \\
    Wind intensity         & $\sigma_u, \sigma_v$ / $\sigma_w$ & 0.5 / 0.25\,m/s \\
    Simulation step        & $\Delta t_{\text{sim}}$ & 0.2\,ms \\
    \bottomrule
  \end{tabular}
\end{table}

\section{GPAC Control Architecture}\label{sec:control}
GPAC is a four-layer hierarchical controller assigning each physical degree of freedom to a dedicated control layer, organized by timescale (Fig.~\ref{fig:architecture}). Each agent~$i$ receives only: (i)~a shared trajectory~$p_d(t)$ broadcast before flight, (ii)~its own IMU/GPS/barometer measurements, and (iii)~local cable tension $T_i$ and direction $q_i$. The payload position is neither measured nor broadcast; each agent reconstructs it locally from its own cable geometry (Section~\ref{sec:estimation}). No payload state, other cable states, or centralized coordinator is required. The control cascade requires no inter-agent state exchange; the only runtime communication is a 10\,Hz neighbor-position broadcast used by the collision-avoidance barrier, and the desired heading $\psi_d$ is part of the pre-broadcast trajectory parameters. Throughout, \emph{layer} refers to the four cascade stages---L1 position/anti-swing, L2 attitude, L3 mass estimation, L4 ESO---while the ESKF sensor fusion and the priority CBF filter are supporting overlays that wrap the cascade rather than additional cascade layers.

\begin{figure}[t!]
\centering
\resizebox{\columnwidth}{!}{%
\begin{tikzpicture}[
  font=\footnotesize,
  box/.style={draw, semithick, rounded corners=2pt, align=center,
              minimum width=3.0cm, minimum height=0.8cm, inner sep=3pt},
  ctrl/.style={box, fill=blue!10},
  cbf/.style ={box, fill=red!10},
  est/.style ={box, fill=teal!12},
  fuse/.style={box, fill=green!10},
  plant/.style={box, fill=black!8, minimum width=7.7cm},
  sig/.style ={-{Stealth[length=1.9mm]}, semithick},
  fb/.style  ={-{Stealth[length=1.9mm]}, semithick, densely dashed},
  sens/.style={-{Stealth[length=1.9mm]}, semithick, densely dashed, black!50},
  lbl/.style ={font=\scriptsize, inner sep=1.5pt, fill=white},
]
  \node[ctrl] (L1) {Layer 1: Position\\+ Anti-Swing (50\,Hz)};
  \node[cbf,  below=0.72cm of L1]  (CBF) {CBF Safety Filter\\(200\,Hz)};
  \node[ctrl, below=0.72cm of CBF] (L2)  {Layer 2: Attitude\\on $\mathrm{SO}(3)$ (200\,Hz)};
  \node[est,  right=1.7cm of L1]  (L3)   {Layer 3: Mass\\Estimator (50\,Hz)};
  \node[est,  right=1.7cm of CBF] (ESO)  {Layer 4: ESO\\($\omega_0\!=\!8$)};
  \node[fuse, right=1.7cm of L2]  (ESKF) {ESKF Sensor\\Fusion (200\,Hz)};
  \node[plant] (plant) at ($(L2.south)!0.5!(ESKF.south)+(0,-0.72)$)
       {Quadrotor$_i$ + Bead-Chain Cable + Payload};
  \draw[sig] (L1)  -- node[lbl,left]{$F_{\mathrm{des}},\,R_d$} (CBF);
  \draw[sig] (CBF) -- node[lbl,left]{$f_{\mathrm{safe}},\,R_d^{\mathrm{safe}}$} (L2);
  \draw[sig] (L2.south) -- node[lbl,left]{$\tau_i,\,f_i$} (L2.south |- plant.north);
  \draw[sig] (L3)  -- node[lbl,above]{$\hat{\theta}_i$} (L1);
  \draw[sig] (ESO) -- node[lbl,above]{$\hat{d}_i$} (CBF);
  \draw[sig] (ESKF) -- node[lbl,right]{$\hat{p}_i,\hat{v}_i$} (ESO);
  \draw[sens] (plant.north -| ESKF.south)
        -- node[lbl,right,text=black!50]{sensors} (ESKF.south);
  \coordinate (fbx) at ($(ESKF.west)+(-0.5,0)$);
  \draw[fb] (ESKF.west) -- (fbx) |- ([yshift=-0.22cm]L1.east);
  \node[lbl, anchor=east] at ($(fbx)!0.22!(fbx|-L1.east)$) {$\hat{p}_i,\hat{v}_i$};
\end{tikzpicture}%
}
\caption{GPAC per-agent architecture. Blue: control cascade; red: CBF overlay;
teal: estimation; green: sensor fusion. Solid arrows are control/estimate
signals; dashed arrows are state feedback (black) and raw sensing (gray). Each
agent executes independently; only neighbor positions are broadcast at
$10\,$Hz (for collision avoidance).}
\label{fig:architecture}
\end{figure}
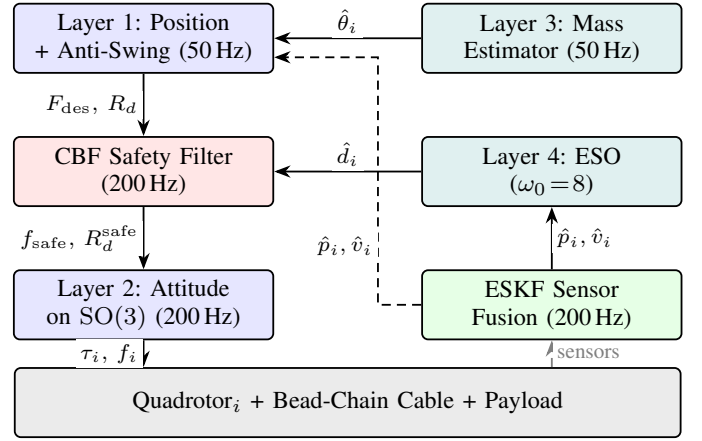

The singularity-free error functions are defined as:
\begin{align}
  e_{R_i} &= \tfrac{1}{2}(R_{d_i}^\top R_i - R_i^\top R_{d_i})^\vee, \label{eq:eR}\\
  \Psi_{q_i} &= 1 - q_{d_i} \cdot q_i \in [0,2], \label{eq:psi_q}\\
  e_{q_i} &= P(q_i)\,q_{d_i} \in T_{q_i}\Sph^2, \label{eq:eq}
\end{align}
where $e_{R_i}$ is the attitude error on $\SOthree$~\cite{lee2010geometric}, $\Psi_{q_i}$ is the cable configuration error on $\Sph^2$~\cite{lee2018geometric}, and $e_{q_i}$ is the negative gradient of $\Psi_{q_i}$ restricted to the tangent space $T_{q_i}\Sph^2$. These are used throughout the control and safety layers that follow.

\subsection{Layer 1: Position Tracking and Anti-Swing Control}

The outermost layer ($\sim$50\,Hz) computes the total desired force as a sum of feedback, feedforward, cable compensation, anti-swing, and ESO terms:
\begin{equation*}
  F_{\text{des},i} = F_{\text{fb},i} + F_{\text{ff},i} + F_{\text{cable},i} + F_{\text{swing},i} + F_{\text{eso},i}.
\end{equation*}
PID tracking feedback with per-axis anti-windup is
$F_{\text{fb},i} = -K_p\,e_{p_i} - K_d\,e_{v_i} - K_i\,e_{I_i}$,
with gains $K_p\!=\!\diag(26,26,24)$, $K_d\!=\!\diag(13,13,12)$, $K_i\!=\!\diag(0.4,0.4,2.5)$ (Table~\ref{tab:params_ctrl}). The vertical integral gain is raised so the integrator cancels the load-induced altitude deficit; the horizontal gain is kept small to avoid windup on the moving reference. Feedforward compensates gravity, the nominal trajectory, and the measured cable tension: $F_{\text{ff},i} = m_Q(a_{d_i} + ge_3) + T_i\hat{n}_i$.

Cable tension compensation prevents treating cable force as disturbance:
$F_{\text{cable},i} = \kappa_T(t)\,T_i\,q_i$,
where $\kappa_T(t) = \min\!\bigl(1,\,(t - t_{\text{taut},i})/\Delta_{\text{ramp}}\bigr)$,
with $t_{\text{taut},i}$ the time cable~$i$ reaches $T_i \geq 1.0$\,N and $\Delta_{\text{ramp}} = 2$\,s. The ramped gain uses only \emph{measured} cable force---no knowledge of $m_L$ or $N$ is required---preventing impulsive loading during the slack-to-taut transition while remaining coordinator-free.

Anti-swing control suppresses pendular oscillations on $\Sph^2$:
$F_{\text{swing},i} = k_q\,e_{q_i} - k_\omega\,\omega_{q_i}$,
with $k_q\!=\!2.0$, $k_\omega\!=\!3.0$~\cite{lee2018geometric}. Here $e_{q_i}$~\eqref{eq:eq} steers $q_i \to q_{d_i}$ along the geodesic and $\omega_{q_i} \in T_{q_i}\Sph^2$ is the cable angular rate, estimated from a dirty derivative of the measured cable direction. Together, these drive $\Psi_{q_i}$ monotonically toward zero in the absence of external perturbation. The ESO disturbance estimate (Layer~4) feeds forward as $F_{\text{eso},i} = m_Q\hat{d}_i$.

Desired rotation $R_{d_i} \in \SOthree$ is extracted from $F_{\text{des},i}$ by aligning the body $z$-axis with the thrust direction~\cite{lee2010geometric}:
\begin{equation}
  b_{3c} = \frac{F_{\text{des},i}}{\norm{F_{\text{des},i}}},\;\; b_{2c} = \frac{b_{3c} \times b_{1d}}{\norm{b_{3c} \times b_{1d}}},\;\; b_{1c} = b_{2c} \times b_{3c},
  \label{eq:Rd_extract}
\end{equation}
with $b_{1d} = [\cos\psi_d,\, \sin\psi_d,\, 0]^\top$ setting the desired heading. The construction is well-defined provided $F_{\text{des},i}\neq 0$ (guaranteed by the gravity feedforward) and $b_{3c}\nparallel b_{1d}$; when $\norm{b_{3c}\times b_{1d}}$ drops below a threshold (near-vertical thrust aligned with the heading), $b_{1d}$ is replaced by the world $x$-axis projected onto the plane orthogonal to $b_{3c}$, keeping $R_{d_i}$ continuous.

\subsection{Layer 2: Geometric Attitude Control on SO(3)}

The inner attitude loop (200\,Hz) uses the geometric tracking controller~\cite{lee2010geometric}:
\begin{equation}
  \tau_i = -k_R\,e_{R_i} - k_\Omega\,e_{\Omega_i} + \Omega_i \times J\Omega_i,
  \label{eq:att_ctrl}
\end{equation}
with $k_R\!=\!8.0$, $k_\Omega\!=\!1.5$, and $e_{\Omega_i} = \Omega_i - R_i^\top R_{d_i}\Omega_{d_i}$. The first two terms provide PD control on $\SOthree$; the third compensates gyroscopic coupling. Feedforward terms $\hatmap{\Omega_i}R_i^\top R_{d_i}\Omega_{d_i} - R_i^\top R_{d_i}\dot{\Omega}_{d_i}$ are omitted under the assumption $\Omega_{d_i} \approx 0$, valid when desired thrust varies slowly relative to the 200\,Hz attitude bandwidth.

\begin{proposition}[Almost-global exponential stability~\cite{lee2010geometric}]\label{prop:lyap}
Define the Lyapunov function
\begin{equation}
  V_R = \tfrac{1}{2}k_R\,\Psi_{R_i} + \tfrac{1}{2}e_{\Omega_i}^\top J\,e_{\Omega_i} + c\,e_{R_i}^\top J\,e_{\Omega_i},
  \label{eq:VR}
\end{equation}
where $\Psi_{R_i} = \frac{1}{2}\tr(I - R_{d_i}^\top R_i)$ is the attitude configuration error and $c>0$ is sufficiently small. With $\Psi_{R_i}(0)<2$:
\begin{equation}
  \dot{V}_R \leq -\lambda_{\min}(W)\bigl(\norm{e_{R_i}}^2 + \norm{e_{\Omega_i}}^2\bigr),
  \label{eq:VR_dot}
\end{equation}
for a positive-definite $W = W(k_R, k_\Omega, J)$, yielding exponential convergence on the dense open set $\{\Psi_R < 2\} \subset \SOthree$.
\end{proposition}

\subsection{Layers 3--4: ESO and Mass Estimator}

\textit{Layer~4} is a third-order Extended State Observer (ESO) per translational axis, modeling dynamics as a double integrator with lumped disturbance $\ddot{p}_k = b_0 u_k + d_k(t)$, where $u_k$ is the commanded acceleration and $b_0 = m_Q/m_{\text{eff}} \approx 0.6$ uses the load-augmented effective mass $m_{\text{eff}} = m_Q + m_L/N$:
\begin{equation}
  \dot{\hat{z}}_{1k} = \hat{z}_{2k} + 3\omega_0\tilde{x}_k,\;\; \dot{\hat{z}}_{2k} = \hat{z}_{3k} + 3\omega_0^2\tilde{x}_k + b_0 u_k,\;\; \dot{\hat{z}}_{3k} = \omega_0^3\tilde{x}_k,
  \label{eq:eso}
\end{equation}
with $\tilde{x}_k = x_k - \hat{z}_{1k}$ and a triple pole at $-\omega_0 = -8$\,rad/s (settling $\sim$0.5\,s). The disturbance estimate $\hat{d}_i = [\hat{z}_{3x}, \hat{z}_{3y}, \hat{z}_{3z}]^\top$ is saturated to $|\hat{d}_{i,k}| \leq 20$\,m/s$^2$ and fed forward via $F_{\text{eso},i}$. Under bounded $\dot{d}$, estimation error scales as $O(\sup|\dot{d}|/\omega_0)$~\cite{guo2013active}.

\textit{Layer~3} adaptively estimates the payload mass share $\hat{\theta}_i \approx m_L/N$ via concurrent learning~\cite{chowdhary2010concurrent}, providing gravity compensation $\hat{\theta}_i(ge_3 + \ddot{p}_d^L)$ to Layer~1.

The cascade exploits timescale separation: the attitude loop ($k_R/J \approx 200$\,rad/s) is far faster than the position loop ($\omega_p = \sqrt{K_p/m_Q} \approx 5$\,rad/s), the ESO ($\omega_0 = 8$\,rad/s), and the mass adaptation (the slowest subsystem), so the attitude dynamics form the fast boundary layer. The ESO bandwidth is reduced from a higher initial setting so its disturbance estimate captures the slow exogenous wind rather than differentiating estimator noise. Singular perturbation arguments guarantee composite stability when the attitude-to-translation bandwidth ratio exceeds a computable threshold~\cite{khalil2002nonlinear}. The hierarchy also provides fault containment: mass estimation errors (Layer~3) affect only feedforward in Layer~1 without corrupting attitude tracking; ESO transients are filtered before reaching attitude; CBF interventions are bounded by the tilt constraint (Section~\ref{sec:safety}). This modularity supports independent layer verification~\cite{koopman2017autonomous}.

\section{Decentralized Adaptive Estimation}\label{sec:estimation}
\begin{figure}[t!]
  \centering
  \includegraphics[width=0.84\columnwidth]{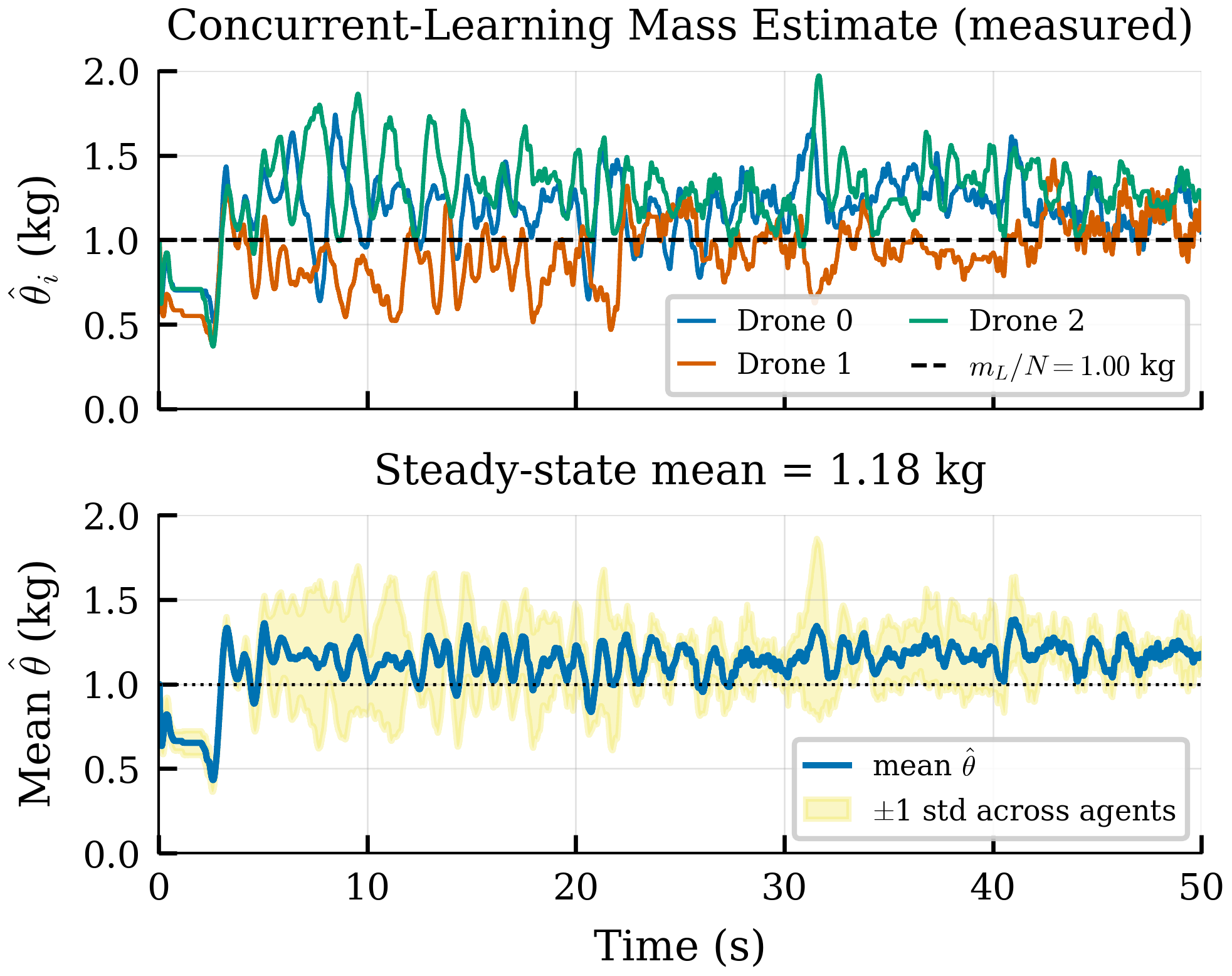}
  \caption{Adaptive mass estimation. \textit{Top:} Measured per-quadrotor concurrent-learning estimates $\hat{\theta}_i$ about the true share $m_L/N = 1.0$\,kg. \textit{Bottom:} Mean across agents with $\pm 1$ standard-deviation band; the buffer-driven update holds the estimate near $m_L/N$ (steady-state mean $1.18$\,kg) without persistent excitation.}
  \label{fig:mass_convergence}
  \vspace{-5pt}
\end{figure}

Each quadrotor runs three estimators in a downward cascade---no payload, cable, or adaptive-parameter states are exchanged: (i)~ESKF for navigation, (ii)~geometric load-state filter, and (iii)~concurrent learning mass estimator. Timescale separation---navigation at 200\,Hz versus load/mass estimation at 50\,Hz---allows each layer to treat inputs as quasi-static, with upstream errors propagating unidirectionally, simplifying stability analysis.

A 15-state ESKF~\cite{sola2017quaternion} with $\delta x = [\delta p,\,\delta v,\,\delta\theta,\,\delta b_a,\,\delta b_g]^\top \in \R^{15}$ fuses IMU (200\,Hz) and GPS (10\,Hz). Propagation uses bias-corrected specific force $a_W = R(\bar{q})(\tilde{a} - b_a) - ge_3$ and angular velocity $\omega_c = \tilde{\omega} - b_g$. Covariance propagation uses the error-state Jacobian: $P \leftarrow \Phi P\Phi^\top + Q_d$. The ESKF provides $(\hat{p}_i, \hat{v}_i, \hat{R}_i)$ to all downstream layers; control closes through the estimator, not ground truth.

When cable~$i$ is taut, payload position is estimated geometrically: $z_{p_i} = \hat{p}_i - L_i\,n_i$, where $n_i \in \Sph^2$ points from quadrotor to payload and $L_i$ is cable rest length. Each quadrotor maintains a Kalman filter with state $[\hat{p}_{L,i}^\top, \hat{v}_{L,i}^\top]^\top \in \R^6$. The prediction uses a constant-velocity model with damping ($\beta_v = 0.05$), biasing the payload velocity toward the quadrotor velocity. This value follows from the quasi-static timescale: $\beta_v = \Delta t_{\text{est}}/\tau_{\text{qs}}$ where $\Delta t_{\text{est}} = 0.02$\,s (50\,Hz) and $\tau_{\text{qs}} = 0.4$\,s is the payload pendular half-period $\approx\!\pi\sqrt{L/g}$. Halving or doubling $\beta_v$ changes load RMSE by $<$5\% in Monte Carlo tests, confirming low sensitivity. Measurement noise is modulated by tension confidence:
\begin{equation}
  \xi_T = \min\!\left(1,\,T_i/T_{\text{conf}}\right), \quad R_k \leftarrow R_k / (0.1 + 0.9\,\xi_T),
  \label{eq:tconf}
\end{equation}
with $T_{\text{conf}} = 20$\,N. When the cable is slack, measurement variance inflates, and the filter relies on prediction. An outlier gate ($\kappa_\nu = 3.0$) prevents cable whip from corrupting estimates. Under Gaussian innovation, the $3\sigma$ gate corresponds to a $\chi^2(3)$ threshold with false-rejection probability $P(\chi^2(3) > 9) = 0.029$; in practice, cable whip produces innovations exceeding $10\sigma$, so the gate reliably rejects corrupted measurements while passing $>$97\% of valid updates.

With a single cable, payload position is observable only along the cable direction; the tangential component relies on a constant-velocity prediction and drifts. The resulting per-agent load estimate is therefore poor, but it is \emph{not} in the tracking loop---it feeds only the mass estimator (Layer~3), which is tracking-neutral (Section~\ref{sec:results})---so its error does not propagate to payload tracking. Fusing all $N$ cables would recover the tangential direction but requires inter-agent communication; this distributed fusion is left to future work.

From the per-cable vertical force equilibrium, each quadrotor constructs a scalar parametric model:
\begin{equation}
  \underbrace{T_i \cos\zeta_i}_{\varphi_i} = \underbrace{\norm{g\,e_3 + \hat{a}_L}}_{Y_i} \cdot \underbrace{\frac{m_L}{N}}_{\theta} + \varepsilon_i,
  \label{eq:regressor}
\end{equation}
where $\zeta_i = \arccos(-n_{i,z})$ is the cable angle from vertical, $\hat{a}_L$ is payload acceleration (via numerical differentiation with low-pass filter, $\tau_f = 0.1$\,s), and $\varepsilon_i$ captures the modeling error. Its dominant component is the cable self-weight: the load cell at the quadrotor end supports the entire cable below it, so $T_i\cos\zeta_i$ also carries the cable weight in addition to the payload share. Unequal cable lengths also redistribute load across agents; this contributes to $\varepsilon_i$ and, ultimately, to the bounded error. Although $\hat{a}_L$ is differentiated from the poor single-cable load estimate, it enters only through $Y_i = \norm{g\,e_3 + \hat{a}_L} \approx g$, which is dominated by gravity, so that the error is masked and does not corrupt the identification.

\begin{remark}[Bias sources in $\hat{\theta}$]\label{rem:eps_bound}
The dominant term is \textbf{cable self-weight}. Because the load cell supports the whole cable below it, $T_i\cos\zeta_i$ includes the cable weight $m_{\text{rope}}g$ ($m_{\text{rope}}=0.2$\,kg per cable) on top of the payload share. Since gravity is vertical, the full cable weight enters the vertical balance $T_i\cos\zeta_i = (m_L/N + m_{\text{rope}})g$ with no $\cos$ factor, giving a payload-independent offset $\varepsilon_i^{\text{cable}} \approx m_{\text{rope}}g \approx 1.96$\,N, i.e.\ a bias $|\tilde{\theta}^{\text{cable}}| \approx m_{\text{rope}} \approx 0.20$\,kg. The observed $\approx +0.18$\,kg at the nominal payload (Table~\ref{tab:massshare}) sits just below this because the bead chain is not perfectly vertical and carries its own dynamics, so $\sim$90\% of the cable weight registers at the top load cell. Two smaller residuals remain. \textbf{Cable asymmetry} ($L_i = \bar{L}(1+\delta_i)$) redistributes load across agents but, by vertical equilibrium, leaves the cross-agent sum $\sum_i T_i\cos\zeta_i$ fixed; it therefore affects the per-agent \emph{spread}, not the mean, and is bounded by $|\varepsilon_i^{\text{asym}}| \leq (m_Lg/N)|\delta_i|\tan\bar{\zeta} \leq 1.28$\,N for $\delta_{\max}=0.19$. \textbf{Acceleration-estimation error} through the low-pass filter ($\tau_f=0.1$\,s) contributes $|\varepsilon_i^{\text{dyn}}| \leq O(\tau_f\norm{\dddot{p}_L}) \leq 0.4$\,N. The UUB radius $|\tilde{\theta}| \leq \bar{\varepsilon}/\bar{Y}$, with $\bar{Y}=\norm{g\,e_3+\hat{a}_L}\approx g$, is dominated by the cable-weight offset; as the payload grows the cables run more vertical and the residual fraction shrinks, so the net per-agent bias falls from $+0.18$ to $+0.02$\,kg across $m_L\in\{3,6,9\}$\,kg (Table~\ref{tab:massshare}).
\end{remark}

The regressor~\eqref{eq:regressor} formalizes the implicit coordination stated in~\eqref{eq:force_convergence}. With feedforward $F_{\text{ff},i} = \hat{\theta}_i \cdot u$ where $u := g\,e_3 + \ddot{p}_L^d$, each estimate converges to the effective per-agent share including its own cable, $\hat{\theta}_i \to m_L/N + m_{\text{rope}}$, so the total feedforward $\textstyle\sum_{i=1}^{N} F_{\text{ff},i}$ converges to $(m_L + Nm_{\text{rope}})\,u$, the total weight the team must support. Including the cable mass in the identified load is in fact correct for feedforward, since the quadrotors lift the cables as well as the payload.
No agent requires knowledge of $N$ or $m_L$; each estimates its share from local cable measurements, and summation yields correct collective compensation. This holds because the regressor $Y_i$ and control $u$ derive from the same shared trajectory, ensuring scalar parametric consistency across agents.

\begin{remark}[Trajectory synchronization]\label{rem:sync}
Since the trajectory is a polynomial stored locally and evaluated analytically, agents compute $u$ identically up to floating-point precision ($\sim\!10^{-15}$ relative error). Any clock offset $\Delta t$ introduces feedforward error bounded by $\norm{\partial u/\partial t}\Delta t$.
\end{remark}

The concurrent-learning update augments gradient descent with stored data:
$\dot{\hat{\theta}}_i = -\gamma\,Y_i\,s_{\text{proj}} - \gamma\rho \sum_{j=1}^{M_i} Y_j(Y_j\hat{\theta}_i - \varphi_j)$, where $\gamma = 0.5$, $\rho = 0.5$, and $s_{\text{proj}} = s_i^\top n_i$ projects sliding variable $s_i = \dot{e}_{L,i} + \lambda e_{L,i}$ ($\lambda = 1.0$) onto cable direction~\cite{chowdhary2010concurrent}. The first term drives online gradient descent; the second replays stored pairs $\{(Y_j, \varphi_j)\}_{j=1}^{M_i}$. The history buffer holds up to $\bar{M} = 50$ points, admitting samples only when excitation exceeds $Y_{\min} = 0.5$\,m/s$^2$ and the regressor differs from the buffer mean by $\delta_Y = 0.1$, ensuring diversity. Estimates are projected to $[\theta_{\min}, \theta_{\max}] = [0.1, 50.0]$\,kg.

\begin{proposition}[Finite-excitation convergence]\label{prop:cl}
If the history stack contains at least one informative sample, $\Sigma_Y = \frac{1}{M}\sum_j Y_j^2 > 0$, the parameter error decays exponentially at rate $\gamma\rho\sum_j Y_j^2 = \gamma\rho M\Sigma_Y$: $\abs{\tilde{\theta}(t)} \leq \abs{\tilde{\theta}(0)}\exp\!\bigl(-\gamma\rho(\textstyle\sum_j Y_j^2)\,t\bigr)$, without persistent excitation of the online trajectory. With bounded modeling error $|\varepsilon_i| \leq \bar{\varepsilon}$, convergence is uniformly ultimately bounded: $|\tilde{\theta}| \leq \bar{\varepsilon}/\bar{Y}$, where $\bar{Y} = \sqrt{\Sigma_Y}$ is the RMS regressor magnitude (consistent with Remark~\ref{rem:eps_bound}).
\end{proposition}

\begin{proof}
Neglecting the online term (rendered nonpositive by projection onto $[\theta_{\min},\theta_{\max}]$~\cite{chowdhary2013exponentially}) and substituting $\varphi_j = Y_j\theta + \varepsilon_j$, the replay term gives $\dot{\tilde{\theta}} = -\gamma\rho(\textstyle\sum_j Y_j^2)\,\tilde{\theta} + \gamma\rho\sum_j Y_j\varepsilon_j$. With $V_\theta = \tfrac12\tilde{\theta}^2$, $\dot{V}_\theta = -\gamma\rho(\sum_j Y_j^2)\tilde{\theta}^2 + \gamma\rho\,\tilde{\theta}\sum_j Y_j\varepsilon_j$. In the noise-free case Gr\"{o}nwall's inequality yields exponential convergence at rate $\gamma\rho\sum_j Y_j^2$; bounded $\varepsilon$ leaves a residual set of radius $|\tilde{\theta}| \leq |\sum_j Y_j\varepsilon_j|/\sum_j Y_j^2 \leq \bar{\varepsilon}/\bar{Y}$ (for $Y_j \approx \bar{Y}$). Empirically, taut-cable samples accumulated during flight hold the estimate within $\approx\pm0.15$\,kg of the effective share (Section~\ref{sec:results}).
\end{proof}

\begin{remark}[Closed-loop estimation--control interaction]\label{rem:closed_loop}
Proposition~\ref{prop:cl} treats the regressor as exogenous. In closed loop, $\hat{\theta}_i$ enters the feedforward $F_{\text{ff},i} = \hat{\theta}_i u$ and so shapes the tracking error $e_{L,i}$ and, in turn, the cable measurements feeding the estimator. A representative gain chain at the steady-state offset $\tilde{\theta}\approx 0.19$\,kg ($\norm{u}\leq 11.8$\,m/s$^2$, $k_{p,\min}=6$, $L\approx 1$\,m): a tracking offset $\lesssim \tilde{\theta}\norm{u}/k_{p,\min} \approx 0.37$\,m, a cable-angle offset $\approx 0.37$\,rad, and a regressor perturbation $\approx 0.37\,T_i \approx 4.5$\,N---which \emph{exceeds} the modeling residual $\bar{\varepsilon}\approx 2.0$\,N. We therefore explicitly do not claim the closed-loop coupling is absorbed within $\bar{\varepsilon}$. (This conservatively counts the cable-weight offset as fed-back error; since the estimator converges to the \emph{effective} share, that offset is in fact compensated, and the residual coupling is smaller.) A rigorous certificate (a contraction condition $\gamma_\theta\gamma_p < 1$ on the estimation--tracking interconnection) is left to future work; empirically, the loop is stable across all tested seeds, team sizes, and payloads (Tables~\ref{tab:massshare},~\ref{tab:nscale}).
\end{remark}

\section{CBF Safety Filter}\label{sec:safety}
\begin{figure}[t!]
  \centering
  \includegraphics[width=0.84\columnwidth]{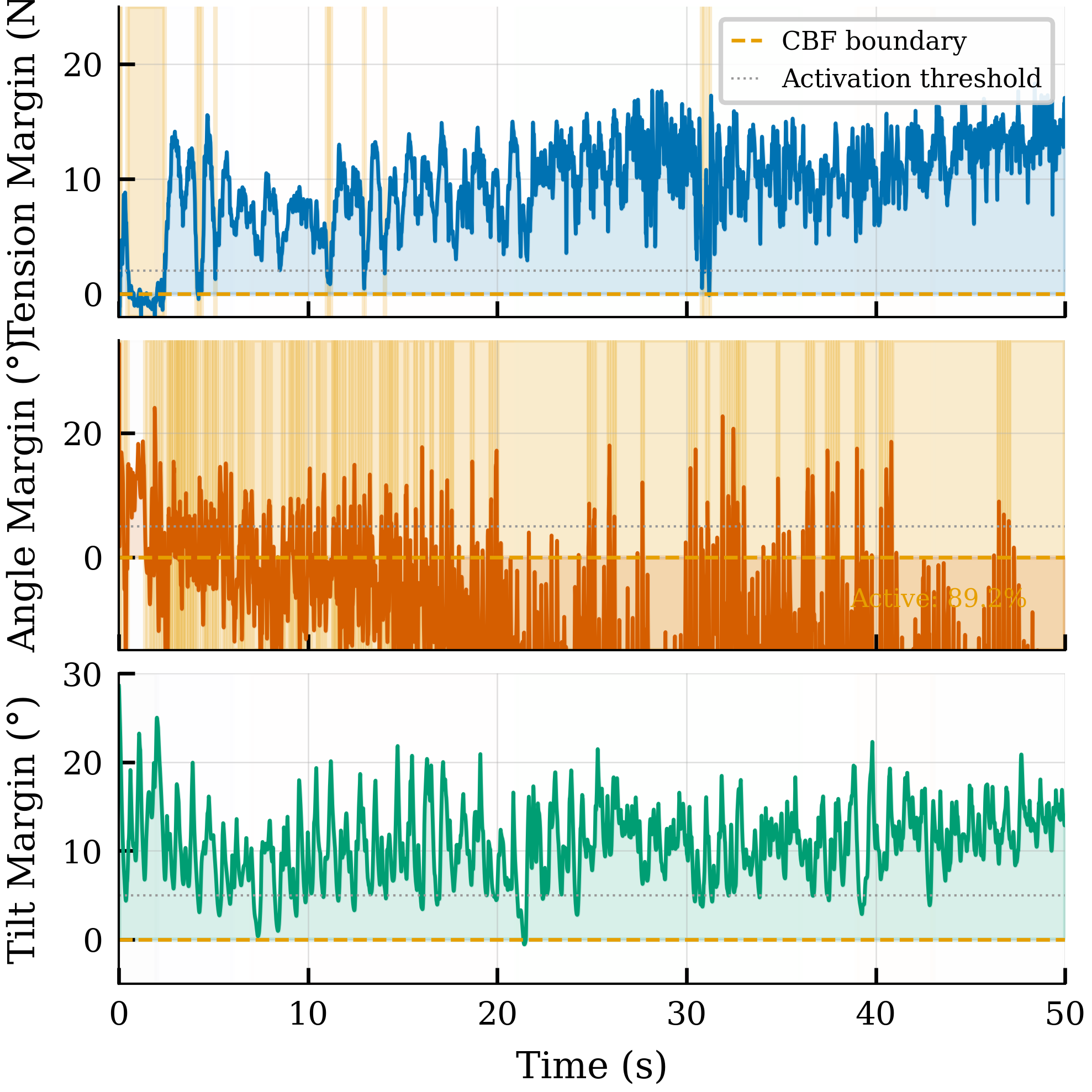}
  \caption{\textit{Top:} Tension margin (distance to $T_{\min}$) for all cables. \textit{Middle:} Cable angle margin (distance to $\zeta_{\max} = 34.4^\circ$); shaded regions indicate CBF activation during aggressive cornering. \textit{Bottom:} Tilt margin, held at the limit. Orange regions denote constraint activation.}
  \label{fig:cbf_activation}
\end{figure}

\begin{figure}[t!]
  \centering
  \includegraphics[width=0.84\columnwidth]{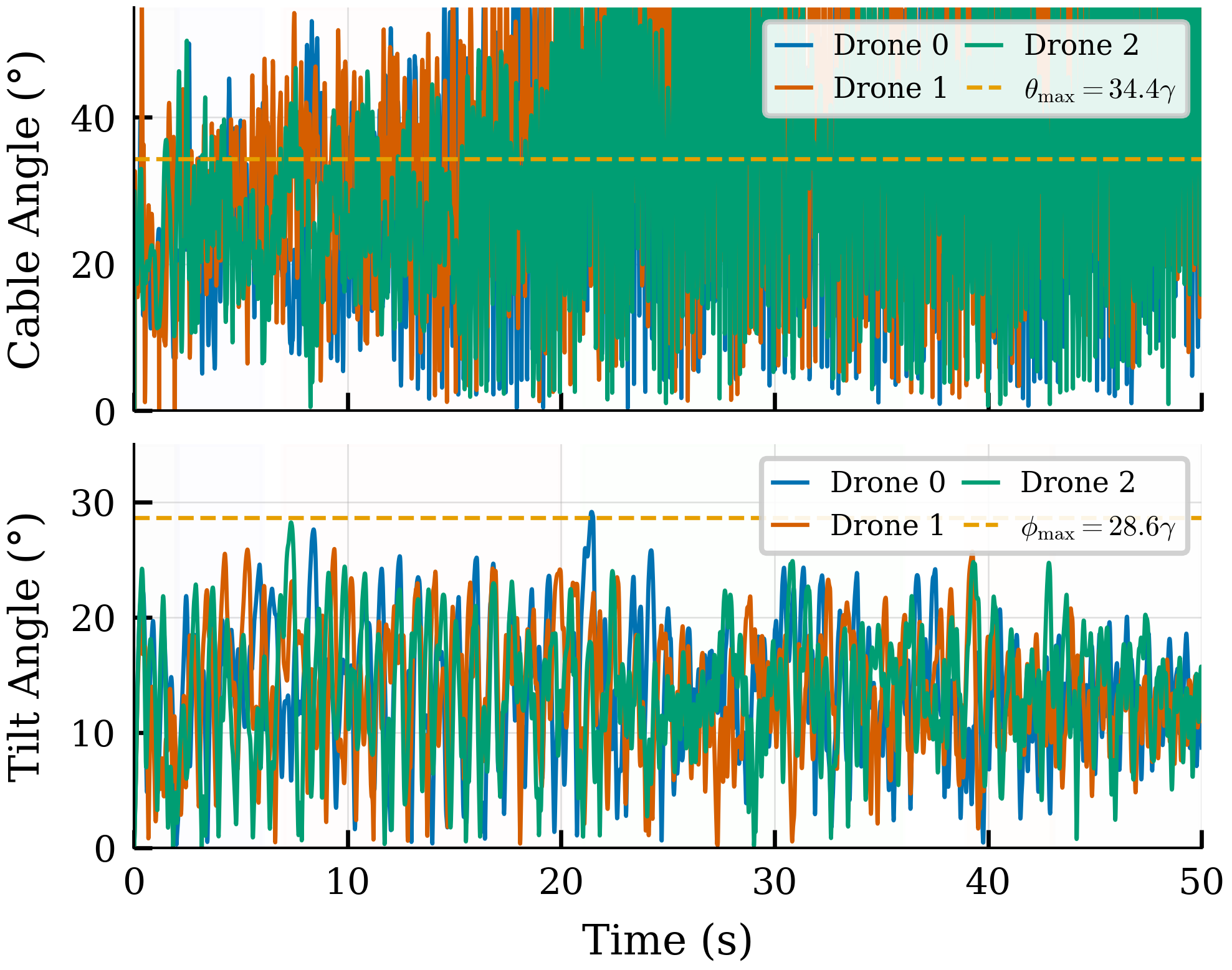}
  \caption{\textit{Top:} Cable angle from vertical for each cable, with the CBF limit $\zeta_{\max} = 34.4^\circ$; this relative-degree-two output exceeds the limit during aggressive cornering. \textit{Bottom:} Quadrotor tilt angle with limit $\phi_{\max} = 28.6^\circ$; the filter holds it at the limit.}
  \label{fig:safety_constraints}
  \vspace{-5pt}
\end{figure}

A modular, priority-ordered CBF-inspired layer supervises operational constraints via minimal post-processing, modifying GPAC output only when necessary and reducing constraint violations rather than strictly enforcing them. Each barrier function corresponds to a transport hazard from Section~\ref{sec:intro}, providing runtime safety supervision within the designed safety envelopes.

The safety filter operates on the force $f_i \in \R^3$ from Layer~1. Abstracting translational dynamics~\eqref{eq:quad_trans} as $m_Q\ddot{p}_i = f_i + w_i(t)$ (affine in control, with lumped disturbance $w_i$ from ESO), five barrier functions encode the safe set $\mathcal{C}_k = \{x \mid h_k(x) \geq 0\}$:
\begin{align}
  h_T^{\text{low}} &= T_i - T_{\min},\;\; h_T^{\text{up}} = T_{\max} - T_i & \text{(tautness)}, \label{eq:h_taut}\\
  h_\zeta &= \cos\zeta_i - \cos\zeta_{\max} & \text{(cable angle)}, \label{eq:h_angle}\\
  h_\omega &= \omega_{\max}^2 - \norm{\omega_{q_i}}^2 & \text{(swing rate)}, \label{eq:h_swing}\\
  h_{\text{tilt}} &= \cos\phi_i - \cos\phi_{\max} & \text{(vehicle tilt)}, \label{eq:h_tilt}\\
  h_{\text{col}} &= \norm{p_i - p_j}^2 - d_{\min}^2 & \text{(collision)}, \label{eq:h_col}
\end{align}
with $T_{\min}\!=\!2$\,N, $T_{\max}\!=\!60$\,N, $\zeta_{\max}\!=\!0.6$\,rad ($34.4^\circ$), $\omega_{\max}\!=\!1.5$\,rad/s, $\phi_{\max}\!=\!0.5$\,rad ($28.6^\circ$), $d_{\min}\!=\!0.8$\,m. Cable angle is denoted $\zeta_i$ to distinguish it from the adaptive parameter $\theta = m_L/N$ (Section~\ref{sec:estimation}). The first two barriers maintain cable tension within $[T_{\min}, T_{\max}]$; $h_\zeta$ and $h_\omega$ limit cable angle and swing rate; $h_{\text{tilt}}$ bounds vehicle roll/pitch; and $h_{\text{col}}$ prevents inter-agent collision. Since $\cos(\cdot)$ is monotonically decreasing on $[0,\pi]$, the safe conditions $\zeta_i \leq \zeta_{\max}$ and $\phi_i \leq \phi_{\max}$ are equivalent to $h_\zeta \geq 0$ and $h_{\text{tilt}} \geq 0$, respectively.

For collision~\eqref{eq:h_col}, the second derivative involves both $f_i$ and $f_j$. In the decentralized setting, agent~$i$ treats neighbor acceleration as a bounded disturbance and enforces a conservative one-sided HOCBF using only its own force, with the disturbance absorbed into the ISSf margin.

Tautness barriers~\eqref{eq:h_taut} are relative-degree-one. Cable angle, tilt, and collision are relative-degree-two, handled via HOCBFs~\cite{xiao2022control, ames2019control}: defining $\psi_1 = \dot{h} + \alpha_1 h$, the constraint $\dot{\psi}_1 + \alpha_2\psi_1 \geq 0$ becomes relative-degree-one in control, with the safe set $\{\psi_1 \geq 0\} \cap \{h \geq 0\}$ forward-invariant. These relative degrees are computed on the force-affine surrogate $m_Q\ddot{p}_i = f_i + w_i$; in the bead-chain plant, tension and cable angle additionally depend on cable stretch, bead dynamics, and slack-to-taut transitions, so the barriers are designed on the surrogate and the high-fidelity simulation evaluates robustness to this model mismatch.

\subsection{Priority-Ordered Safety Projection with Conditional ISSf Margins}

The idealized filter is the QP solved at each control cycle:
\begin{align}
  \label{eq:cbf_qp}
  f_{\text{safe}} = \argmin_{f,\,\delta}\;\norm{f - f_{\text{nom}}}^2 + \lambda\!\sum_j\!\delta_j^2 \\
  \;\;\text{s.t.}\;\; B_j(x,f) \geq -\mu_j - \delta_j,\nonumber
\end{align}
where $B_j$ is the control-affine constraint function for barrier $j$: $\dot{h}_j + \alpha_j h_j$ for the relative-degree-one tautness barriers, and the HOCBF form $\dot{\psi}_{1,j} + \alpha_{2,j}\psi_{1,j}$ (with $\psi_{1,j}=\dot{h}_j+\alpha_{1,j}h_j$) for the relative-degree-two angle, tilt, and collision barriers. With $\lambda = 100$ penalizing relaxation, and since hard forward invariance of $\{h \geq 0\}$ is unrealistic under disturbances, we use Input-to-State Safety (ISSf), coupling the robustness margin to the ESO estimate; i.e., $B_j \geq -\mu_{\text{base}} - \kappa_d\norm{\hat{d}_i}$,
with $\mu_{\text{base}} = 2.0$, $\kappa_d = 1.5$. When the ESO reports large disturbances, the margin grows, and the filter activates earlier. For a relative-degree-one barrier the steady-state violation bound is $h_j(t) \geq -(\mu_{\text{base}} + \kappa_d\bar{d})/\alpha_j$

The implementation in this paper uses sequential gradient projection with priority ordering (tautness> angle> tilt> swing> collision) rather than a full QP. This approximation equals~\eqref{eq:cbf_qp} when at most one constraint is active---typical, as the filter is transparent in nominal flight and engages only near a constraint boundary during aggressive cornering. The formal ISSf bound on $h_j(t)$ therefore holds under single-constraint activation; under simultaneous multi-constraint activation, only priority-ordered feasibility is ensured. The priority ordering reflects failure severity: tautness ranks highest because cable slack causes immediate loss of controllability, while the inter-agent clearance is held at $\geq 0.67$\,m against the $d_{\min} = 0.8$\,m bound. Slack variables $\delta_j$ ensure lower-priority constraints degrade gracefully. Per-agent cost is $O(N_c)$ with $N_c = 6$ constraints ($\sim$1200\,FLOPs/cycle).

The safety filter modifies force direction, changing the desired rotation $R_d$. A critical requirement is that this perturbation not exit the geometric controller's stability region.

\begin{theorem}[Safety--attitude compatibility]\label{thm:compatibility}
While the tilt barrier~\eqref{eq:h_tilt} is active it constrains the safe thrust direction to the cone $\phi_i \leq \phi_{\max}=0.5$\,rad, so the safe desired attitude lies within $\phi_{\max}$ of the hover attitude:
\begin{equation}
  \Psi_R(I, R_d^{\text{safe}}) = 1 - \cos\phi_i \leq 1 - \cos\phi_{\max} \approx 0.12,
  \label{eq:compat}
\end{equation}
($\leq 0.15$ even at the worst observed tilt of $32^\circ$). The CBF-modified attitude reference, therefore, stays deep within the hover neighborhood---far from the antipodal configuration ($\Psi_R = 2$) that bounds the geometric controller's almost-global basin (Proposition~\ref{prop:lyap})---so the safety filter cannot drive the desired attitude toward the controller's stability boundary.
\end{theorem}

\begin{proof}
The tilt barrier $h_{\text{tilt}} = \cos\phi_i - \cos\phi_{\max} \geq 0$ constrains the safe thrust direction $b_{3c} = f_{\text{safe}}/\norm{f_{\text{safe}}}$~\eqref{eq:Rd_extract} to $\phi_i \leq \phi_{\max}$. Since $R_d^{\text{safe}}$ aligns $b_{3c}$ with the body $z$-axis, its geodesic distance from the hover attitude $I$ (where $b_3 = e_3$) is exactly $\phi_i$, giving $\Psi_R(I, R_d^{\text{safe}}) = 1 - \cos\phi_i \leq 1 - \cos\phi_{\max}$. This bound holds whenever the tilt barrier is active---the regime in which the filter modifies the attitude reference.
\end{proof}

\begin{remark}[Actual attitude]\label{rem:actual_att}
Theorem~\ref{thm:compatibility} bounds the desired \emph{reference}, not the actual attitude $R_i$. The geometric controller (Proposition~\ref{prop:lyap}) drives $R_i \to R_d^{\text{safe}}$ exponentially from any error $\Psi_R(R_i, R_d^{\text{safe}}) < 2$; because the reference stays within $0.12$ of hover and the attitude loop is fast ($200$\,Hz, $\sim$5\,ms settling, Remark~\ref{rem:timescale}), the tracking error remains well inside the basin in simulation. We do not claim exponential tracking of the time-varying safe reference in the presence of the omitted angular-rate feedforward.
\end{remark}

\begin{remark}[Timescale separation]\label{rem:timescale}
Compatibility relies on three timescales: fast attitude ($k_R/J \approx 200$ rad/s, settling ~5 ms), medium safety filter (Butterworth cutoff $2\pi \times 15 \approx 94$ rad/s), and slow position/cable dynamics ($\sqrt{g/L} \approx 3$ rad/s). The attitude-to-CBF bandwidth ratio $200/94 \approx 2.1$ is moderate. Classical singular perturbation~\cite[Ch.~11]{khalil2002nonlinear} typically requires this ratio to be much greater than one. Instead, we rely on a practical composability argument. The attitude settling time (~5 ms) is less than the CBF Butterworth delay (~11 ms). This guarantees the attitude controller resolves any step change in $R_d^{\text{safe}}$ within one safety filter cycle. Butterworth rate-limiting ensures $\norm{\dot{R}_d^{\text{safe}}} \leq 2\pi f_c \cdot 2\phi_{\max}$. Thus, the $\Omega_{d_i} \approx 0$ simplification in~\eqref{eq:att_ctrl} remains valid. The stronger separation $\omega_{\text{att}}/\omega_{\text{pos}} \approx 200/3 \approx 67$ does satisfy formal singular perturbation, ensuring the position-attitude cascade is stable. The CBF operates within this well-separated hierarchy.
\end{remark}

\section{Simulation Results}\label{sec:results}
We validated GPAC using a Drake-based~\cite{drake2024} multibody simulation, with parameters listed in Table~\ref{tab:params}. The physics engine runs at 5000 Hz using a semi-implicit Euler method, and cable vibrations (around 55 Hz) are resolved with 90 times oversampling. The cable rest lengths are asymmetric ([0.994, 1.155, 0.952] m for the baseline seed), showing up to 19\% variation. The multi-rate timing is set to match the GPAC layer structure. The simulation environment includes Dryden wind and a full sensor suite. All control loops are closed using the ESKF rather than ground truth, ensuring sensor-in-the-loop realism. The principal controller, estimator, and CBF parameters are listed in Table~\ref{tab:params_ctrl}; the complete parameter set, the waypoint trajectory, and the per-seed randomization are released with the open-source implementation.\footnote{Anonymized for review; the repository and commit hash will be provided in the camera-ready version.} We report the $3$D payload-position RMSE against the piecewise-linear load reference over the post-pickup window $t\ge 2$\,s, as the mean $\pm$ standard deviation (and coefficient of variation, CV) over the Monte-Carlo seeds.

\begin{table}[t]
  \centering
  \caption{Principal GPAC controller, ESO, concurrent-learning (CL), and CBF parameters. The full set is in the open-source implementation.}
  \label{tab:params_ctrl}
  \footnotesize
  \setlength{\tabcolsep}{4pt}
  \begin{tabular}{@{}llll@{}}
    \toprule
    \textbf{Param.} & \textbf{Value} & \textbf{Param.} & \textbf{Value} \\
    \midrule
    $k_p^{xy},k_d^{xy}$ & $26,\,13$ & $k_p^{z},k_d^{z}$ & $24,\,12$ \\
    $k_i^{xy},k_i^{z}$  & $0.4,\,2.5$ & $\phi_{\max}$ & $0.5$\,rad \\
    $k_q,k_\omega$ (swing) & $2.0,\,3.0$ & $k_R,k_\Omega$ (att.) & $8.0,\,1.5$ \\
    $\omega_o,b_0$ (ESO) & $8$\,rad/s,\,$0.6$ & $\tau_{\mathrm{ESO}}$ & $0.1$\,s \\
    $\gamma,\rho$ (CL) & $0.5,\,0.5$ & buffer $\bar{M}$ & $50$ \\
    $\alpha_T,\alpha_\zeta$ (CBF) & $1.5,\,1.0$ & $\alpha_\phi,\alpha_c$ & $2.0,\,5.0$ \\
    $d_{\min}$ & $0.8$\,m & broadcast & $10$\,Hz \\
    \bottomrule
  \end{tabular}
  \vspace{-6pt}
\end{table}

The overall 3D RMSE is 33.6 cm for the baseline seed (4.8\% of the workspace diagonal and within one cable length); a 13-seed Monte Carlo gives a mean of 33.8 cm with a 2.8\% coefficient of variation, so performance is essentially seed-invariant. This is achieved with all control and estimation loops closed through the ESKF (no ground-truth feedback), under wind disturbance and 19\% cable asymmetry. The horizontal and vertical channels contribute comparably (22.3 and 25.1 cm): the payload trails and swings as the formation drags it through the figure-eight, while the vertical error is concentrated in the pickup transient before the altitude integrator settles. The largest excursions occur during the figure-eight cornering, where the payload pendulum is most excited.

\begin{figure}[t]
  \centering
  \includegraphics[width=0.84\columnwidth]{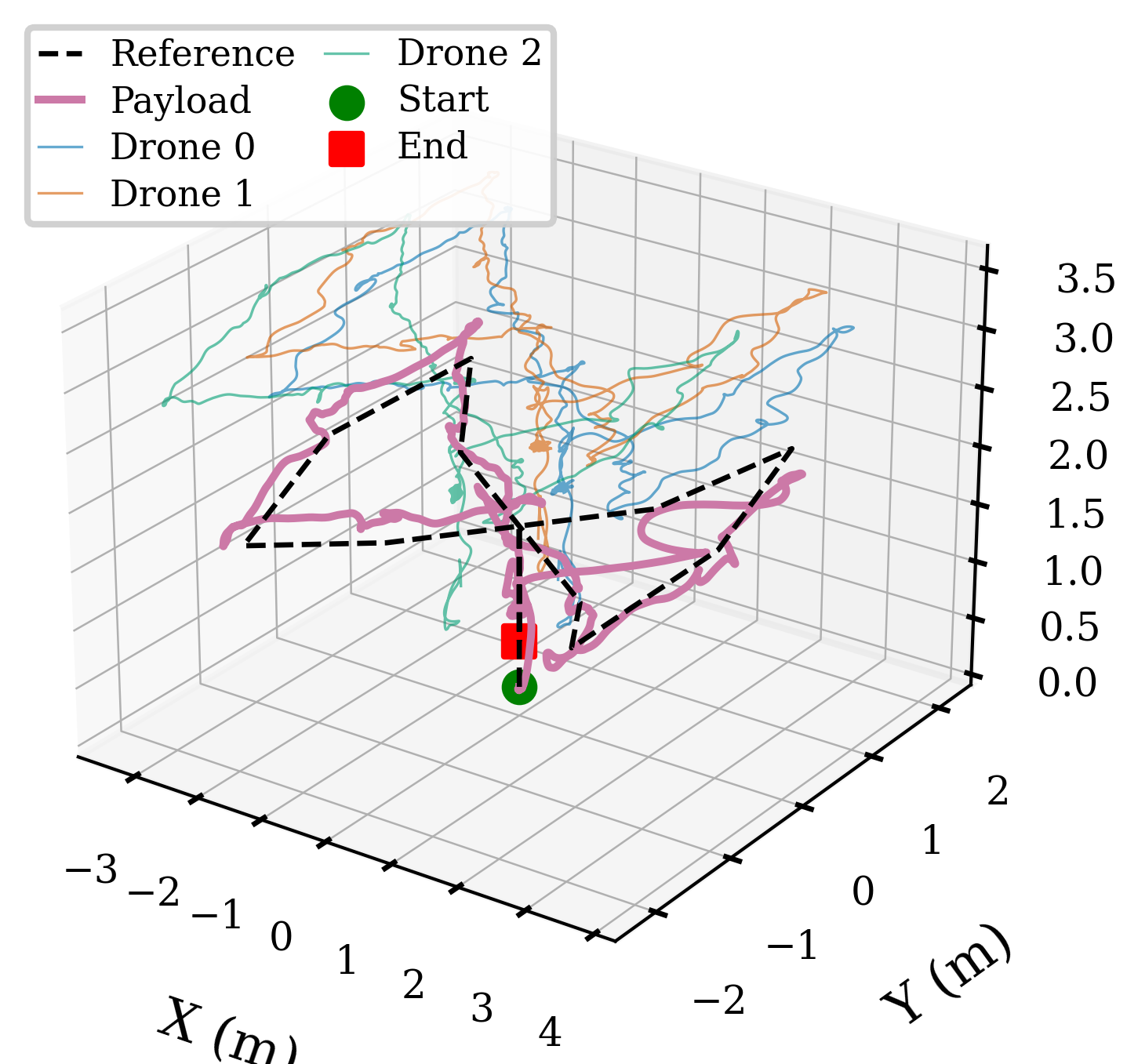}
  \caption{3D trajectory of the load and the quadcopters. The max velocity of the load is 1\,m/s.}
  \label{fig:trajectory_3d}
\end{figure}

\begin{table}[t]
  \centering
  \caption{Payload tracking errors (cm) for the baseline seed; the 13-seed Monte-Carlo mean is $33.8\pm0.9$ cm (2.8\% CV).}
  \label{tab:tracking}
  \begin{tabular}{@{}lcccccc@{}}
    \toprule
    & \multicolumn{2}{c}{\textbf{Horiz.}} & \multicolumn{2}{c}{\textbf{Vert.}} & \multicolumn{2}{c}{\textbf{3D}} \\
    \textbf{Phase} & RMSE & Max & RMSE & Max & RMSE & Max \\
    \midrule
    Ascent (2--6\,s) & 15.9 & 42.3 & 47.1 & 60.1 & 49.8 & 61.0 \\
    Fig-8 right (7--20\,s) & 26.6 & 65.9 & 16.8 & 41.2 & 31.5 & 73.2 \\
    Fig-8 left (24--36\,s) & 20.7 & 38.1 & 21.2 & 29.2 & 29.6 & 45.2 \\
    Descent (39--43\,s) & 14.2 & 30.2 & 24.2 & 33.1 & 28.1 & 38.3 \\
    Post-descent (43--50\,s) & 6.7 & 11.7 & 29.0 & 33.7 & 29.8 & 33.7 \\
    \midrule
    \textbf{Overall} & 22.3 & 72.2 & 25.1 & 60.1 & \textbf{33.6} & 74.5 \\
    \bottomrule
  \end{tabular}
\end{table}

\begin{figure}[t]
  \centering
  \includegraphics[width=0.84\columnwidth]{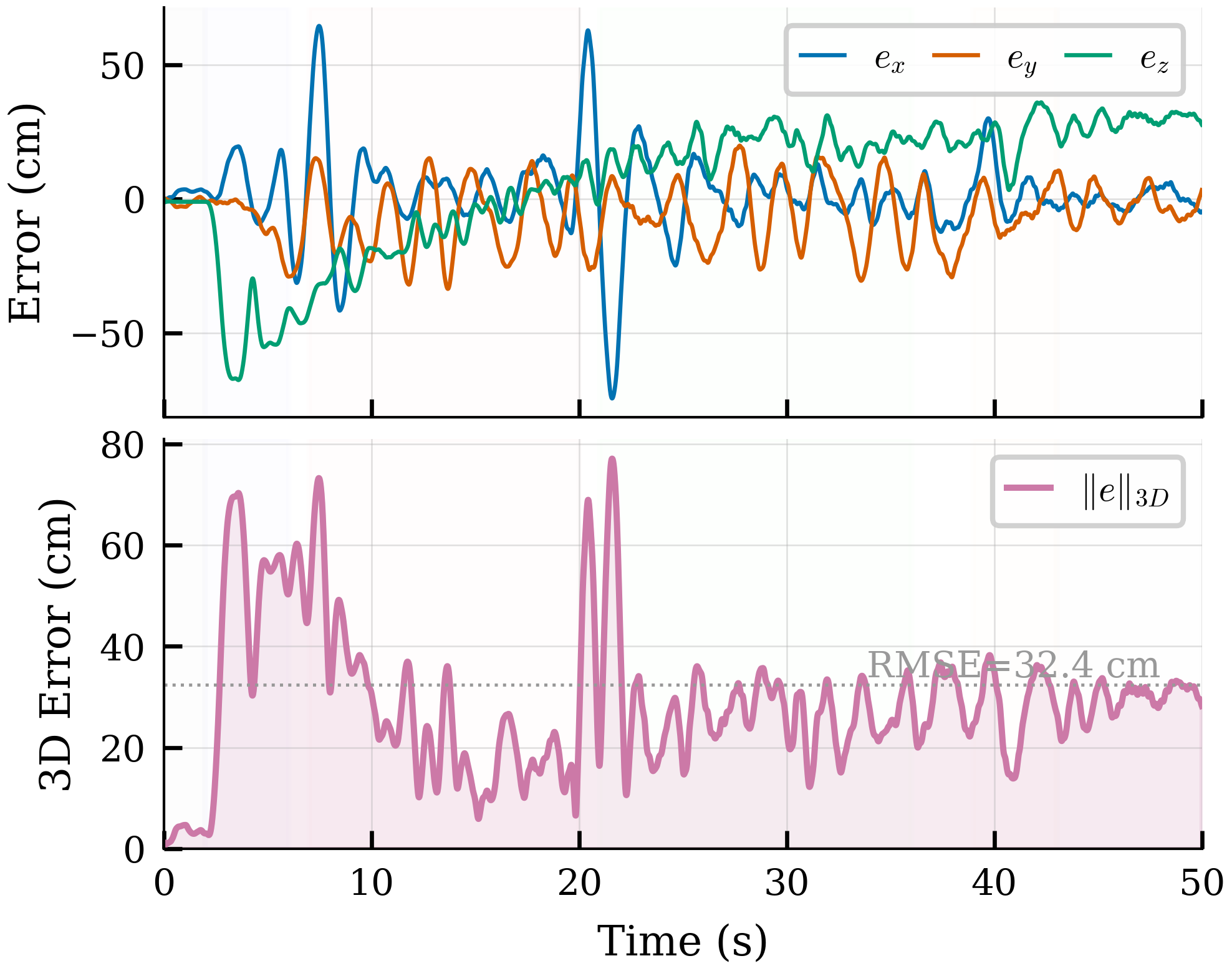}
  \caption{Payload tracking error. \textit{Top:} 3D error (RMSE 33.6\,cm); peaks during aggressive cornering. \textit{Bottom:} Components. The horizontal and vertical channels contribute comparably; the vertical error is concentrated in the pickup transient.}
  \label{fig:tracking_error}
  \vspace{-5pt}
\end{figure}

\begin{figure}[t]
  \centering
  \includegraphics[width=0.84\columnwidth]{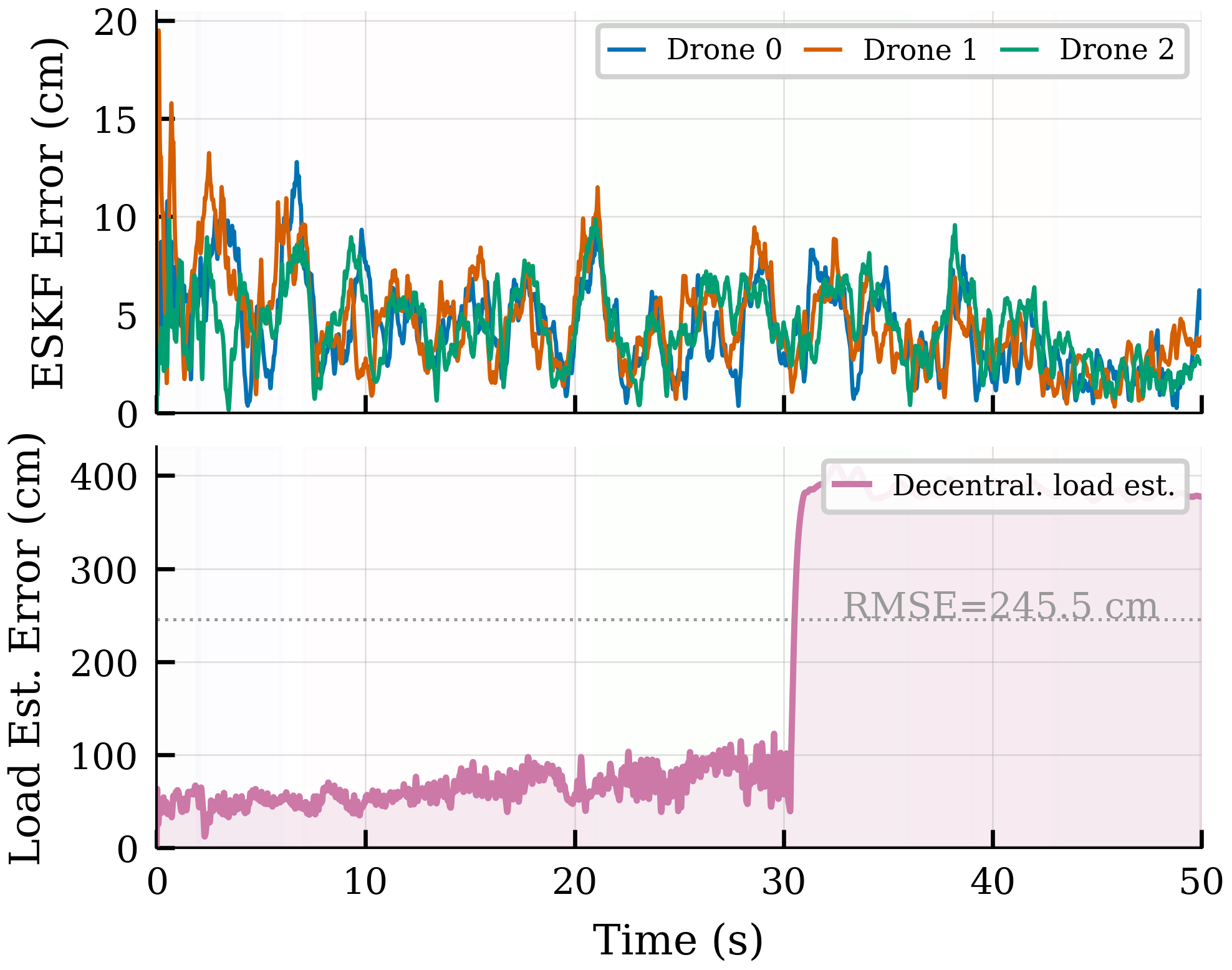}
  \caption{Estimation performance. \textit{Top:} ESKF error (5.0\,cm RMSE). \textit{Middle:} Decentralized load estimate, limited by single-cable observability. \textit{Bottom:} Per-quadrotor concurrent-learning mass estimates fluctuating about $m_L/N=1.0$\,kg (steady-state mean $1.18$\,kg).}
  \label{fig:estimation_error}
  \vspace{-5pt}
\end{figure}

Table~\ref{tab:tracking} breaks the baseline-seed error down by phase. The horizontal channel is tightest during the post-descent hover (6.7 cm) and largest through the figure-eight (26.6 cm) as the formation drags the payload through the corners. The vertical error is dominated by the ascent phase (47.1 cm), the slack-to-taut pickup transient before the altitude integrator cancels the load-induced deficit; in steady cruise, it settles to roughly 17--24 cm.

The ESKF gives a quadrotor position RMSE of about 5.0 cm, near the GPS noise floor, with errors increasing during cornering. Low variance across Monte Carlo seeds indicates that ESKF accuracy depends on sensor noise rather than cable geometry. The decentralized single-cable load-position estimate is poor (RMSE $\approx 2.5$\,m, Fig.~\ref{fig:estimation_error}), dominated by the unobservable tangential direction, with a step near $t\approx30$\,s when the figure-eight reverses and the tangential offset flips sign. Crucially, this estimate feeds only the tracking-neutral mass estimator and not the position loop, so it does not affect payload tracking ($33.8$\,cm); recovering the tangential direction via distributed multi-cable fusion is left to future work.

The concurrent-learning estimator admits taut-cable regressor samples into a fixed buffer and keeps the per-agent mass share close to the true $m_L/N = 1.0$\,kg throughout the flight (Fig.~\ref{fig:mass_convergence}), with a steady-state mean of $1.18$\,kg and a $0.01$\,kg spread across seeds; the buffer-driven update provides this without persistent excitation. The impact of removing it is quantified in Table~\ref{tab:ablation}.

To separate the estimator's identification role from any tracking role, we sweep the payload over $m_L\in\{3,6,9\}$\,kg (per-agent share $1.0$--$3.0$\,kg, $N=3$), running each case with the adaptive feedforward active (full) and disabled (no-CL); Table~\ref{tab:massshare} reports both. Two facts emerge. First, the estimate recovers the true share across the full $3\times$ range, with an absolute bias that stays within the Remark~\ref{rem:eps_bound} bound $\bar\varepsilon/\bar Y\approx0.2$\,kg and \emph{shrinks} with load ($+0.18$ down to $+0.02$\,kg): as the share grows the cables run more vertical and the additive modeling residual becomes a smaller fraction of $T_i\cos\zeta_i$. Second, the tracking RMSE is statistically identical with and without the adaptive feedforward at every mass---the per-axis altitude integrator absorbs the static-load deficit regardless of which feedforward supplies it. The mass-share estimate is therefore an \emph{identification} mechanism---the quantity that enables implicit coordination~\eqref{eq:force_convergence} without communication---rather than a tracking gain.

\begin{table}[t]
  \centering
  \caption{Mass-share identification across a $3\times$ payload sweep ($N=3$, 3 seeds each). $\hat\theta$ tracks the true share $m_L/N$ with absolute bias within the Remark~\ref{rem:eps_bound} bound ($\bar\varepsilon/\bar Y\approx0.2$\,kg); tracking RMSE is unchanged whether the estimate drives the feedforward (full) or not (no-CL).}
  \label{tab:massshare}
  \footnotesize
  \setlength{\tabcolsep}{4pt}
  \begin{tabular}{@{}cccccc@{}}
    \toprule
    $\boldsymbol{m_L}$ & \textbf{Share} & $\boldsymbol{\hat\theta}$ & \textbf{Bias} & \textbf{RMSE full} & \textbf{RMSE no-CL} \\
    \textbf{(kg)} & \textbf{(kg)} & \textbf{(kg)} & \textbf{(kg)} & \textbf{(cm)} & \textbf{(cm)} \\
    \midrule
    3 & 1.0 & $1.18\pm0.04$ & $+0.18$ & $31.6\pm1.4$ & $32.0\pm1.7$ \\
    6 & 2.0 & $2.11\pm0.09$ & $+0.11$ & $30.7\pm0.4$ & $31.4\pm0.6$ \\
    9 & 3.0 & $3.02\pm0.11$ & $+0.02$ & $33.3\pm0.4$ & $33.4\pm0.1$ \\
    \bottomrule
  \end{tabular}
  \vspace{-6pt}
\end{table}

Table~\ref{tab:failure_modes} shows the hazard-to-mitigation mapping and the measured constraint performance, and makes explicit which constraints are strictly held and which are only reduced. The tension filter maintains positive cable tension in steady flight (per-cable mean $13$--$17$\,N), with brief slack-side excursions to $\approx 0.7$\,N (below $T_{\min}=2$\,N) in the worst seed during the most aggressive cornering. The collision barrier, fed a $10$\,Hz neighbor-position broadcast, keeps the inter-agent clearance at $0.85$\,m on average, with a worst-case $0.67$\,m that dips just below the $d_{\min} = 0.8$\,m bound; the cable swing rate stays strictly within its $1.5$\,rad/s bound ($1.3$\,rad/s observed). The quadrotor tilt rides its $28.6^\circ$ cone and slightly exceeds it under load (mean $29^\circ$, worst $32^\circ$). The most-exceeded constraint is the cable angle, which reaches $48^\circ$ (limit $34.4^\circ$) during aggressive cornering, as the force-level filter has limited single-step authority over this relative-degree-two output. In summary, the filter holds the swing rate strictly and keeps tension, clearance, and tilt near their bounds with brief excursions, but does not strictly enforce the cable angle; reducing these residual violations---e.g.\ with explicit higher-order barriers on the swing and cable-angle dynamics---is left to future work.

Table~\ref{tab:ablation} isolates each layer's contribution with every loop closed through the estimator. The ESO disturbance feedforward is the dominant tracking benefit: removing it more than doubles the RMSE ($+110\%$) and lets the cable angle reach $68^\circ$, since the formation no longer rejects the Dryden wind. The CBF improves tracking by $19\%$ \emph{and} enforces the safety constraints---removing it both raises the error and collapses the worst-case inter-agent clearance to $0.32$\,m, confirming that the barrier, not the tracking controller, produces the separation. The concurrent-learning feedforward is near-neutral for tracking ($+0\%$): as the mass sweep (Table~\ref{tab:massshare}) confirms, the per-axis altitude integrator absorbs the static-load deficit at every payload, so the adaptive estimate's role is the explicit per-agent mass share (Fig.~\ref{fig:mass_convergence}) that enables implicit coordination, rather than a tracking gain. We note that the ESO is driven by a low-pass-filtered position estimate at a reduced observer bandwidth ($\omega_o = 8$\,rad/s); a higher bandwidth on the raw estimate instead differentiates measurement noise into the disturbance channel and erodes the very benefit the ablation reports.

\begin{table}[t]
  \centering
  \caption{Ablation results: payload tracking RMSE (cm).}
  \label{tab:ablation}
  \begin{tabular}{@{}lccc@{}}
    \toprule
    \textbf{Configuration} & \textbf{RMSE} & $\boldsymbol{\Delta}$ & \textbf{Max Angle} \\
    \midrule
    Full GPAC (baseline) & 33.6 & --- & 49$^\circ$ \\
    No concurrent learning & 33.6 & $+$0\% & 50$^\circ$ \\
    No ESO feedforward & 70.6 & $+$110\% & 68$^\circ$ \\
    No CBF safety filter & 40.1 & $+$19\% & 46$^\circ$ \\
    \bottomrule
  \end{tabular}
\end{table}

\begin{table}[t]
  \centering
  \caption{Cable tension statistics (N) during steady-state flight. Asymmetric cable lengths produce unequal load sharing.}
  \label{tab:tension_stats}
  \begin{tabular}{@{}lccccc@{}}
    \toprule
    \textbf{Cable} & $\boldsymbol{L_i}$ \textbf{(m)} & \textbf{Mean} & \textbf{Std} & \textbf{Min} & \textbf{Max} \\
    \midrule
    0 & 0.994 & 16.0 & 3.4 & 5.4 & 29.4 \\
    1 & 1.155 & 13.2 & 4.2 & 1.9 & 26.9 \\
    2 & 0.952 & 17.1 & 3.4 & 5.9 & 31.9 \\
    \bottomrule
  \end{tabular}
\end{table}

\begin{figure}[t]
  \centering
  \includegraphics[width=0.84\columnwidth]{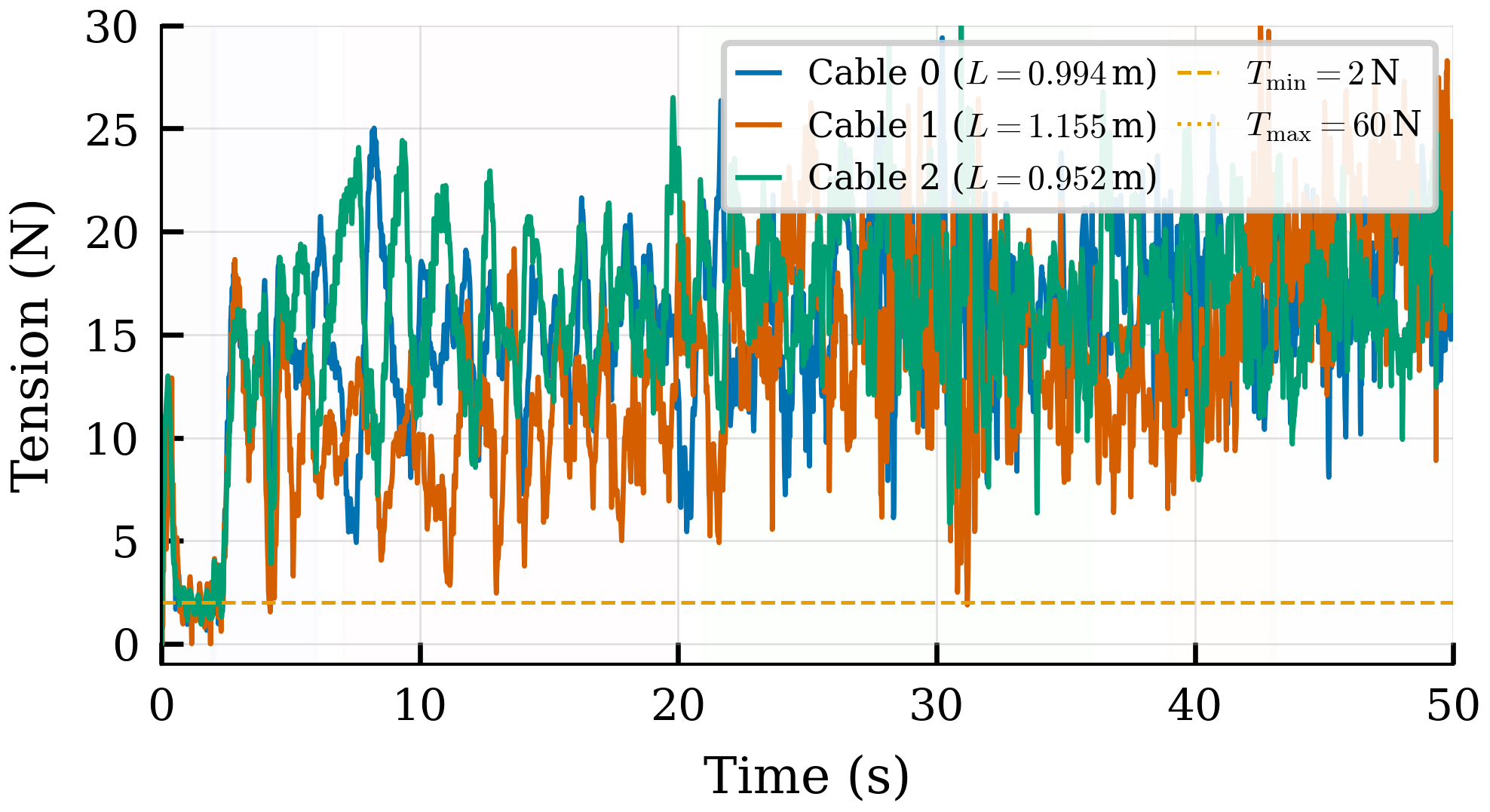}
  \caption{Cable tensions. Asymmetric lengths yield unequal load sharing, accommodated without coordination.}
  \label{fig:cable_tensions}
  \vspace{-5pt}
\end{figure}

\begin{table}[t]
  \centering
  \caption{Hazard-to-mitigation mapping with constraint limits and measured performance.}
  \label{tab:failure_modes}
  \begin{tabular}{@{}llccc@{}}
    \toprule
    \textbf{Hazard} & \textbf{Mitigation} & \textbf{Limit} & \textbf{Observed}\\
    \midrule
    Mass unknown & CL estimation & --- & $\hat{\theta}\!=\!1.18\!\pm\!0.01$\,kg\\
    Wind disturbance & ESO feedfwd & --- & $+110\%$ w/o (Tab.~\ref{tab:ablation})\\
    Cable slack & CBF tension & {[2, 60]\,N} & min $\approx 0.7$\,N (worst)\\
    Cable angle & CBF + anti-swing & $34.4^\circ$ & $48^\circ$\\
    Excessive tilt & CBF tilt & $28.6^\circ$ & $29^\circ$\\
    Swing rate & CBF swing & 1.5\,rad/s & 1.3\,rad/s\\
    Collision & CBF separation & 0.8\,m & $\geq 0.67$\,m\\
    Sensor noise & ESKF fusion & --- & 5.0\,cm RMSE\\
    \bottomrule
  \end{tabular}
\end{table}

\subsection{Team-Size Scaling}

To test team-size invariance directly, we run the \emph{identical} per-agent stack---no gains, geometry, or trajectory retuned---for $N\in\{2,\dots,6\}$, holding the per-agent share constant at $m_L/N = 1.0$\,kg (i.e.\ $m_L = N$\,kg) so that team size is decoupled from per-agent load. Each entry is a 3-seed mean (Table~\ref{tab:nscale}).

\begin{table}[t]
  \centering
  \caption{Team-size scaling at constant per-agent share ($m_L/N=1.0$\,kg). Identical per-agent controller; 3 seeds each. The nominal adjacent-agent spacing $2r\sin(\pi/N)$ at the fixed formation radius $r=0.6$\,m is the geometric driver of the $N\!\ge\!5$ degradation.}
  \label{tab:nscale}
  \footnotesize
  \setlength{\tabcolsep}{4pt}
  \begin{tabular}{@{}cccccc@{}}
    \toprule
    $\boldsymbol{N}$ & \textbf{RMSE (cm)} & \textbf{Clear.\ (m)} & \textbf{Tilt} & $\boldsymbol{\hat{\theta}}$ \textbf{(kg)} & \textbf{Spacing (m)} \\
    \midrule
    2 & $32.6\pm1.5$ & $0.96$ & $28.8^\circ$ & $1.21$ & $1.20$ \\
    3 & $31.6\pm1.4$ & $0.87$ & $29.7^\circ$ & $1.18$ & $1.04$ \\
    4 & $31.2\pm1.4$ & $0.49$ & $30.3^\circ$ & $1.27$ & $0.85$ \\
    5 & $51.5\pm1.3$ & $0.21$ & $37.0^\circ$ & $1.09$ & $0.71$ \\
    6 & $62.0\pm3.3$ & $0.12$ & $39.5^\circ$ & $1.14$ & $0.60$ \\
    \bottomrule
  \end{tabular}
  \vspace{-4pt}
\end{table}

For $N\in\{2,3,4\}$ the payload RMSE is flat at $31$--$33$\,cm and the per-agent mass-share estimate tracks its $1.0$\,kg target (within the same $\sim$$20\%$ bias seen at $N=3$): the identical local control law is \emph{structurally reusable} across team size, and tracking performance is team-size-invariant. Safety margins, however, degrade earlier than tracking, and the controller's reusability should not be read as safe scaling. The collision clearance stays above $d_{\min}=0.8$\,m only for $N\le3$ ($0.87$\,m at $N=3$); by $N=4$ it has already fallen to $0.49$\,m, and the quadrotor tilt sits marginally above the nominal $28.6^\circ$ cone at every $N$ (consistent with the baseline result, where the tilt rides its limit). The cause is formation geometry, not the control law: with a fixed radius $r=0.6$\,m the nominal adjacent-agent spacing $2r\sin(\pi/N)$ shrinks monotonically ($1.04$\,m at $N=3$, $0.85$ at $N=4$, $0.71$ at $N=5$, $0.60$ at $N=6$), and the realized clearance falls faster still as the loaded agents lean inward, so the collision filter is increasingly asked to maintain a separation the formation cannot geometrically provide and instead trades it for tilt. By $N\ge5$ the nominal spacing itself drops below $d_{\min}$ and tracking collapses. Restoring clearance requires scaling $r$ with $N$, which for the fixed $1$\,m cable length splays the cables past the angle limit $\zeta_{\max}=34.4^\circ$; larger teams therefore require proportionally longer cables. In short, the per-agent controller scales without modification, but the fixed-radius formation provides reliable inter-agent clearance only up to $N=3$ (worst-case $0.67$\,m even there); beyond that, the formation radius and cable length---not the controller---set the practical team-size ceiling.

\section{Conclusion}\label{sec:conclusion}
This paper presented GPAC, a four-layer hierarchical controller for decentralized cooperative aerial transport on $\SEthree \times (\Sph^2)^N$. The architecture's central property is that each quadrotor runs an identical control stack using only local sensors and its cable, yet the collective system achieves coordinated payload transport with subsystem-level stability guarantees and priority-ordered safety supervision. No knowledge of the team size $N$ or payload mass $m_L$ is required; no payload, cable, or adaptive states are exchanged between agents---each reconstructs the payload position locally from its own cable---and the only inter-agent communication is a low-rate neighbor-position broadcast for collision avoidance.

The hazard-oriented decomposition---each layer targeting a specific failure mode (Table~\ref{tab:failure_modes})---enables three properties that are individually well-studied but rarely combined: geometric manifold-based stability, decentralized adaptive learning, and runtime safety supervision. Hierarchical timescale separation limits cross-layer fault propagation, the absence of a central coordinator removes that single point of failure, and the per-agent control law is structurally reusable across team sizes (Table~\ref{tab:nscale}). High-fidelity simulation with flexible cables, onboard sensor fusion, and wind turbulence---with every loop closed through the ESKF---confirms the architecture operates closed-loop at low per-agent computational cost, achieving a 33.8 cm mean payload-tracking RMSE while maintaining positive cable tension (except brief slack-side excursions), the inter-agent clearance above $0.67$\,m at $N=3$, the swing rate within bounds, and the tilt at its limit. Two constraints are not strictly enforced: the cable angle, a relative-degree-two output, is briefly exceeded during aggressive cornering, and the inter-agent clearance degrades as team size grows at fixed formation radius (Section~\ref{sec:results}). Bounding the cable angle with a higher-order barrier on the swing dynamics, and scaling the formation geometry with $N$, are left to future work.

The present results are simulation-only, making hardware flight experiments the immediate next step. Two architectural limitations motivate further algorithmic development: the decentralized single-cable load-position estimate is limited by observability (though it lies outside the tracking loop and thus does not affect payload tracking), and the sequential CBF projection guarantees only priority-ordered feasibility—rather than full ISSf—when multiple constraints activate simultaneously. Accordingly, future work targets distributed consensus-based estimation to recover the unobservable tangential direction while preserving communication efficiency.

\bibliographystyle{IEEEtran}
\bibliography{References}

\end{document}